\documentclass{article}

\usepackage{microtype}
\usepackage{graphicx}
\usepackage{subfigure}
\usepackage{booktabs} 
\usepackage{xcolor}
\usepackage{amsthm}
\usepackage{multirow}

\usepackage{hyperref}

\usepackage{amsmath,amsfonts,bm,amssymb}

\newtheorem{theorem}{Theorem}
\newtheorem{definition}{Definition}

\def\eqref#1{equation~\ref{#1}}

\def\1{\bm{1}}

\def\vtheta{{\bm{\theta}}}
\def\va{{\bm{a}}}

\def\vg{{\bm{g}}}
\def\vh{{\bm{h}}}

\def\vq{{\bm{q}}}
\def\vr{{\bm{r}}}

\def\vz{{\bm{z}}}
\def\vpi{{\bm{\pi}}}
\def\vomega{{\bm{\omega}}}

\def\vtheta{{\bm{\theta}}}
\def\vphi{{\bm{\phi}}}

\DeclareMathAlphabet{\mathsfit}{\encodingdefault}{\sfdefault}{m}{sl}
\SetMathAlphabet{\mathsfit}{bold}{\encodingdefault}{\sfdefault}{bx}{n}

\def\gA{{\mathcal{A}}}

\def\gD{{\mathcal{D}}}

\def\gP{{\mathcal{P}}}

\def\gS{{\mathcal{S}}}

\newcommand{\E}{\mathbb{E}}

\newcommand{\R}{\mathbb{R}}

\newcommand{\KL}{D_{\mathrm{KL}}}

\DeclareMathOperator*{\argmax}{arg\,max}

\newcommand{\PL}{\mathrm{PL}}
\newcommand{\LSBRE}{\mathrm{LSBRE}}
\newcommand{\expreturn}{\mathrm{ExpRet}}

\usepackage[accepted]{icml2019}

\icmltitlerunning{Multi-Agent Adversarial Inverse Reinforcement Learning}

\begin{document}

\twocolumn[
\icmltitle{Multi-Agent Adversarial Inverse Reinforcement Learning}




\begin{icmlauthorlist}
\icmlauthor{Lantao Yu}{st}
\icmlauthor{Jiaming Song}{st}
\icmlauthor{Stefano Ermon}{st}
\end{icmlauthorlist}

\icmlaffiliation{st}{Department of Computer Science, Stanford University, Stanford, CA 94305 USA}

\icmlkeywords{Machine Learning, ICML}

\vskip 0.3in
]


\icmlcorrespondingauthor{Lantao Yu}{lantaoyu@cs.stanford.edu}
\icmlcorrespondingauthor{Stefano Ermon}{ermon@cs.stanford.edu}
\printAffiliationsAndNotice{}  


\begin{abstract}
Reinforcement learning agents are prone to undesired behaviors due to reward mis-specification. Finding a set of reward functions to properly guide agent behaviors is particularly challenging in multi-agent scenarios. 
Inverse reinforcement learning provides a framework to automatically acquire suitable reward functions from expert demonstrations. Its extension to multi-agent settings, however, is difficult due to the more complex notions of rational behaviors.
In this paper, we propose MA-AIRL, a new framework for multi-agent inverse reinforcement learning, which is effective and scalable for Markov games with high-dimensional state-action space and unknown dynamics. We derive our algorithm based on a new solution concept and maximum pseudolikelihood estimation within an adversarial reward learning framework. In the experiments, we demonstrate that MA-AIRL can recover reward functions that are highly correlated with ground truth ones, and significantly outperforms prior methods in terms of policy imitation.
\end{abstract}

\section{Introduction}
Reinforcement learning (RL) is a general and powerful framework for decision making under uncertainty. Recent advances in deep learning have enabled a variety of RL applications such as games \cite{silver2016mastering, mnih2015human}, robotics \cite{gu2017deep,levine2016end}, automated machine learning \cite{zoph2016neural} and generative modeling \cite{yu2017seqgan}. RL algorithms are also showing promise in multi-agent systems, where multiple agents interact with each other, such as multi-player games~\citep{peng2017multiagent}, social interactions~\cite{leibo2017multi} and multi-robot control systems~\citep{matignon2012coordinated}. 
However, the success of RL crucially depends on careful reward design~\citep{amodei2016concrete}. 
As reinforcement learning agents are prone to undesired behaviors due to reward mis-specification~\citep{amodei2016faulty}, designing suitable reward functions can be challenging in many real-world applications~\citep{hadfield2017inverse}. In multi-agents systems, since different agents may have completely different goals and state-action representations, hand-tuning reward functions becomes increasingly more challenging as we take more agents into consideration. 

Imitation learning presents a direct approach to programming agents with expert demonstrations, where agents learn to produce behaviors similar to the demonstrations. However, imitation learning algorithms, such as behavior cloning \cite{pomerleau1991efficient} and generative adversarial imitation learning \cite{ho2016generative,ho2016model}, typically sidestep the problem of interring an explicit representation for the underlying reward functions.
Because the reward function is often considered as the most succinct, robust and transferable representation of a task~\citep{abbeel2004apprenticeship,fu2017learning}, it is important to consider the problem of inferring reward functions from expert demonstrations, which we refer to as inverse reinforcement learning (IRL). 
IRL can offer many advantages compared to direct policy imitation, such as analyzing and debugging an imitation learning algorithm, inferring agents' intentions and re-optimizing rewards in new environments~\cite{ng2000algorithms}.

However, IRL is ill-defined, as many policies can be optimal for a given reward
and many reward functions can explain a set of demonstrations. 
Maximum entropy inverse reinforcement learning (MaxEnt IRL) \cite{ziebart2008maximum} provides a general probabilistic framework to solve the ambiguity by finding the trajectory distribution with maximum entropy that matches the reward expectation of the experts. As MaxEnt IRL requires solving an integral over all possible trajectories for computing the partition function, it is only suitable for small scale problems with known dynamics. Adversarial IRL \cite{finn2016connection,fu2017learning} scales MaxEnt IRL to large and continuous problems by drawing an analogy between a sampling based approximation of MaxEnt IRL and Generative Adversarial Networks~\cite{goodfellow2014generative} with a particular discriminator structure. However, the approach is restricted to single-agent settings.

In this paper, we consider the IRL problem 
in multi-agent environments with high-dimensional continuous state-action space and unknown dynamics.
Generalizing MaxEnt IRL and Adversarial IRL to multi-agent systems is challenging. Since each agent's optimal policy depends on other agents' policies, the notion of optimality, central to Markov decision processes, has to be replaced by an appropriate equilibrium solution concept. Nash equilibrium \citep{hu1998multiagent} is the most popular solution concept for multi-agent RL, where each agent's policy is the best response to others.
However, Nash equilibrium is incompatible with MaxEnt RL in the sense that it assumes the agents never take sub-optimal actions.
Thus imitation learning and inverse reinforcement learning methods based on Nash equilibrium or correlated equilibrium \cite{aumann1974subjectivity} might lack the ability to handle irrational (or computationally bounded) experts. 

In this paper, inspired by logistic quantal response equilibrium \cite{mckelvey1995quantal,mckelvey1998quantal} and Gibbs sampling \cite{hastings1970monte}, we propose a new solution concept termed logistic stochastic best response equilibrium (LSBRE), which allows us to characterize the trajectory distribution induced by parameterized reward functions and handle the bounded rationality of expert demonstrations in a principled manner. 
Specifically, by uncovering the close relationship between LSBRE and MaxEnt RL, and bridging the optimization of joint likelihood and conditional likelihood with \textit{maximum pseudolikelihood estimation}, we propose Multi-Agent Adversarial Inverse Reinforcemnt Learning (MA-AIRL), a novel MaxEnt IRL framework for Markov games. MA-AIRL is effective and scalable to large high-dimensional Markov games with unknown dynamics, which are not amenable to previous methods relying on tabular representation and linear or quadratic programming \cite{natarajan2010multi,waugh2013computational,lin2014multi,lin2018multi}.
We experimentally demonstrate that MA-AIRL is able to recover reward functions that are highly correlated to the ground truth rewards, while simultaneously learning policies that significantly outperform state-of-the-art multi-agent imitation learning algorithms~\citep{song2018multi} in mixed cooperative and competitive tasks~\citep{lowe2017multi}.
\section{Preliminaries}
\subsection{Markov Games}

Markov games~\citep{littman1994markov} are generalizations of Markov decision processes (MDPs) to the case of $N$ 
interacting agents. 
A \emph{Markov game} $(\gS, \gA, P, \eta, \vr)$ is defined via a set of states $\gS$, and $N$ sets of actions $\{\gA_i\}_{i=1}^{N}$. The function $P: \gS \times \gA_1 \times \cdots \times \gA_N \to \gP(\gS)$ describes the (stochastic) transition process between states, where $\gP(\gS)$ denotes the set of probability distributions over the set $\gS$. 
Given that we are in state $s^t$ at time $t$ and the agents take actions $(a_1, \ldots, a_N)$, the state transitions to $s^{t+1}$ with probability $P(s^{t+1} | s^t, a_1, \ldots, a_N)$.
Each agent $i$ obtains a (bounded) reward given by a function $r_i: \gS \times \gA_1 \times \cdots \times \gA_N \to \R$. The function $\eta \in \gP(\gS)$ specifies the distribution of the initial state. 
We use bold variables without subscript $i$ to denote the concatenation of all variables for all agents
(\emph{e.g.}, $\vpi$ denotes the joint policy, $\vr$ denotes all rewards and $\va$ denotes actions of all agents in a multi-agent setting).
We use subscript $-i$ to denote \emph{all agents except for $i$}. For example, $(a_i, \va_{-i})$ represents $(a_1, \ldots, a_N)$, the actions of all $N$ agents.
The objective of each agent $i$ is to maximize its own expected return (\emph{i.e.}, the expected sum of discounted rewards) $\E_{\vpi}\left[\sum_{t=1}^{T} \gamma^t r_{i,t}\right]$, where $\gamma$ is the discount factor and $r_{i, t}$ is the reward received $t$ steps into the future. 
Each agent achieves its own objective by selecting actions through a stochastic policy
$\pi_i: \gS \to  \gP(\gA_i)$.
Depending on the context, the policies can be Markovian (\emph{i.e.}, depend only on the state) or require additional coordination signals.
For each agent $i$, we further define the expected return for a state-action pair as:
$
\expreturn_i^{\pi_i, \vpi_{-i}} (s_t, \va_t) = \mathbb{E}_{s^{t+1:T}, \va^{t+1:T}}\left[\sum_{l \geq t} \gamma^{l-t} r_i(s^l, \va^l)|s_t, \va_t, \vpi \right]
$

\subsection{Solution Concepts for Markov Games}\label{sec:optimality-notion}
A correlated equilibrium (CE) for a Markov game \cite{ziebart2011maximum} is a joint strategy profile, where no agent can achieve higher expected reward through unilaterally changing its own policy.
CE first introduced by \cite{aumann1974subjectivity,aumann1987correlated} is a more general solution concept than the well-known Nash equilibrium (NE)~\citep{hu1998multiagent}, which further requires agents' actions in each state to be independent, \emph{i.e.} $\vpi(\va|s) = \Pi_{i=1}^N \pi_i(a_i|s)$. 
It has been shown that many decentralized, adaptive strategies will converge to CE instead of a more restrictive equilibrium such as NE \cite{gordon2008no, hart2000simple}.
To take bounded rationality into consideration, \cite{mckelvey1995quantal,mckelvey1998quantal} further propose logistic quantal response equilibrium (LQRE) as a stochastic generalization to NE and CE.
\begin{definition}\label{def:lqre}
A logistic quantal response equilibrium for Markov game corresponds to any strategy profile satisfying a set of constraints, where for each state and action, the constraint is given by:
\begin{gather*}
\pi_i(a_{i} | s) = \frac{ \exp\left( \lambda \expreturn^{\vpi}_{i}(s, a_i, \va_{-i}) \right)}{\sum_{a'_{i}} \exp\left( \lambda \expreturn^{\vpi}_{i}(s, a'_i, \va_{-i})\right)}
\end{gather*}
\end{definition}
Intuitively, in LQRE, agents choose actions with higher expected return with higher probability.

\subsection{Learning from Expert Demonstrations}
Suppose we do not have access to the ground truth reward signal $r$, but have demonstrations $\gD$ provided by an expert ($N$ expert agents in Markov games). $\gD$ is a set of trajectories $\{\tau_j\}_{j=1}^{M}$, where $\tau_j = \{(s^t_j, \va^t_j)\}_{t=1}^{T}$ is an expert trajectory collected by sampling $s^1 \sim \eta(s), \va^t \sim \pi_E(\va^t \vert s^t), s^{t+1} \sim P(s^{t+1} \vert s^t, \va^t)$.
$\gD$ contains the entire supervision to the learning algorithm, \emph{i.e.}, we assume we cannot ask for additional interactions with the experts during training. Given $\gD$, imitation learning (IL) aims to directly learn policies that behave similarly to these demonstrations, whereas inverse reinforcement learning (IRL) \cite{russell1998learning, ng2000algorithms} seeks to infer the underlying reward functions which induce the expert policies.

The MaxEnt IRL framework~\citep{ziebart2008maximum} aims to recover a reward function that rationalizes the expert behaviors with the \emph{least commitment}, denoted as $\mathrm{IRL}(\pi_E)$:
\begin{align*}
\mathrm{IRL}(\pi_E) & = \argmax_{r \in \R^{\gS \times \gA}} \E_{\pi_E}[r(s, a)] - \mathrm{RL}(r) \\    
\mathrm{RL}(r) & = \max_{\pi \in \Pi} \mathcal{H}(\pi) + \E_\pi[r(s, a)]
\end{align*}
where $\mathcal{H}(\pi)=\mathbb{E}_{\pi}[-\log \pi(a|s)]$ is the policy entropy.
However, MaxEnt IRL is generally considered less efficient and scalable than direct imitation, as we need to solve a forward RL problem in the inner loop. In the context of imitation learning,~\citep{ho2016generative} proposed to use generative adversarial training \cite{goodfellow2014generative}, to learn the policies characterized by $\mathrm{RL} \circ \mathrm{IRL}(\pi_E)$ directly, leading to the Generative Adversarial Imitation Learning (GAIL) algorithm:
$$
\min_\theta \max_\omega \mathbb{E}_{\pi_E}\left[\log D_\omega(s,a)\right] + \mathbb{E}_{\pi_\theta}\left[\log(1 - D_\omega(s,a))\right]
$$
where $D_\omega$ is a discriminator that classifies expert and policy trajectories, and $\pi_\theta$ is the parameterized policy that tries to maximize its score under $D_\omega$. 
According to~\citet{goodfellow2014generative}, with infinite data and infinite computation, at optimality, the distribution of generated state-action pairs should exactly match the distribution of demonstrated state-action pairs under the GAIL objective.
The downside to this approach, however, is that we bypass the intermediate step of recovering rewards. Specifically, note that we cannot extract reward functions from the discriminator, as $D_\omega(s, a)$ will converge to $0.5$ for all $(s, a)$ pairs.

\subsection{Adversarial Inverse Reinforcement Learning}\label{sec:single-airl}
Besides resolving the ambiguity that many optimal rewards can explain a set of demonstrations, another advantage of MaxEnt IRL is that it can be interpreted as solving the following maximum likelihood estimation (MLE) problem:
\begin{align}
    & p_\omega(\tau) \propto \left[\eta(s^1) \prod_{t=1}^T P(s^{t+1}|s^t,a^t)\right] \exp\left( \sum_{t=1}^T r_\omega (s^t, a^t) \right) \label{eq:irl-mle}\\
    & \max_\omega \mathbb{E}_{\pi_E}\left[\log p_\omega(\tau)\right] = \mathbb{E}_{\tau \sim \pi_E}\left[ \sum_{t=1}^{T} r_\omega(s^t,a^t) \right] - \log Z_\omega \nonumber
\end{align}

Here, $\omega$ are the parameters of the reward function and $Z_\omega$ is the partition function, \emph{i.e.} an integral over all possible trajectories consistent with the environment dynamics. $Z_\omega$ is intractable to compute when the state-action spaces are large or continuous, and the environment dynamics are unknown.

Combining Guided Cost Learning (GCL) \cite{finn2016guided} and generative adversarial training, \citeauthor{finn2016connection,fu2017learning} proposed adversarial IRL framework as an efficient sampling based approximation to the MaxEnt IRL, where the discriminator takes on a particular form:
$$
D_\omega(s,a) = \frac{\exp(f_\omega(s,a))}{\exp(f_\omega(s,a)) + q(a|s)}
$$
where $f_\omega(s,a)$ is the learned function, $q(a|s)$ is the probability of  the adaptive sampler pre-computed as an input to the discriminator,
and the policy is trained to maximize $\log D - \log(1-D)$.

To alleviate the reward shaping ambiguity \cite{ng1999policy}, where many reward functions can explain an optimal policy, 
\cite{fu2017learning} further restricted $f$ to a reward estimator $g_\omega$ and a potential shaping function $h_\phi$:
$$
f_{\omega, \phi} (s,a,s') = g_\omega(s,a) + \gamma h_\phi(s') - h_\phi(s)
$$
It has been shown that under 
suitable assumptions, 
$g_\omega$ and $h_\phi$ will recover the true reward and value function up to a constant.
\section{Method}


\subsection{Logistic Stochastic Best Response Equilibirum}\label{sec:lsbre}
To extend MaxEnt IRL to Markov games, we need be able to characterize the trajectory distribution induced by a set of (parameterized) reward functions $\{r_i(s, \va)\}_{i=1}^N$ (analogous to Equation~(\ref{eq:irl-mle})). However existing optimality notions introduced in Section~\ref{sec:optimality-notion} do not \textit{explicitly} define a tractable joint strategy profile that we can use to maximize the likelihood of expert demonstrations (as a function of the rewards); they do so \textit{implicitly} as the solution to a set of constraints.  

Motivated by Gibbs sampling \cite{hastings1970monte}, dependency networks \cite{heckerman2000dependency}, best response dynamic \cite{nisan2011best,gandhi2012stochastic} and LQRE, we propose a new solution concept that allows us to characterize rational (joint) policies induced from a set of reward functions. Intuitively, our solution concept corresponds to the result of repeatedly applying a stochastic (entropy-regularized) best response mechanism, where each agent (in turns) attempts to optimize its actions while keeping the other agents' actions fixed.  

To begin with, let us first consider a stateless single-shot normal-form game with $N$ players and a reward function $r_i: \mathcal{A}_1 \times \ldots \times \mathcal{A}_N \rightarrow \mathbb{R}$ for each player $i$. We consider the following \textit{Markov chain} over $(\mathcal{A}_1 \times \cdots \times \mathcal{A}_N)$, where the state of the Markov chain at step $k$ is denoted $\mathbf{z}^{(k)} = (z_1, \cdots, z_N)^{(k)}$,
with each random variable $z^{(k)}_i$ taking values in $\mathcal{A}_i$. The transition kernel of the Markov Chain is defined by the following equations:
\begin{align}
    \label{eq:conditionals}
    z^{(k+1)}_i \sim &~ P_i(a_i|\va_{-i}=\vz^{(k)}_{-i}) = \frac{\exp(\lambda r_i(a_i, \vz^{(k)}_{-i}))}{\sum_{a'_i} \exp(\lambda r_i(a'_i, \vz^{(k)}_{-i}))}
\end{align}
and each agent $i$ is updated in scan order. Given all other players' actions $\vz^{(k)}_{-i}$, the $i$-th player picks an action proportionally to $\exp(\lambda r_i(a_i, \vz^{(k)}_{-i}))$, where $\lambda > 0$ is a parameter that controls the level of rationality of the agents. For $\lambda \rightarrow 0$, the agent will select actions uniformly at random, while for $\lambda \rightarrow \infty$, the agent will select actions greedily (best response).
Because the Markov Chain is ergodic, it admits a unique stationary distribution which we denote $\vpi(\va)$. Interpreting this stationary distribution over $(\mathcal{A}_1 \times \cdots \times \mathcal{A}_N)$ as a policy, we call this stationary joint policy a \textit{logistic stochastic best response equilibrium} for normal-form games.

Now let us generalize the solution concept to Markov games. For each agent $i$, let $\{\pi^t_i\}_{t=1}^T$ denote a set of time-dependent policies. First we define the state action value function for each agent $i$. Starting from the base case:
\begin{align*}
    Q^{\vpi^\varnothing}_i(s^T, a^T_i, \va^T_{-i}) = r_i(s^T, a^T_i, \va^T_{-i})
\end{align*}
then we recursively define:
{\small
\begin{align*}
  Q_i^{\vpi^{t+1:T}} (s^t,a^t_i, &\va^t_{-i}) = r_i(s^t, a^t_i, \va_{-i}^{t}) +\\ &\mathbb{E}_{s^{t+1} \sim P(\cdot|s^t, \va^t)}
 \Big[\mathcal{H}(\pi^{t+1}_i(\cdot|s^{t+1})) + \\ & 
 \mathbb{E}_{\va^{t+1} \sim \vpi^{t+1}(\cdot|s^{t+1})}[ Q_i^{\vpi^{t+2:T}} (s^{t+1}, \va^{t+1})]\Big]
\end{align*}}
which generalizes the standard state-action value function in single-agent RL ($\va_{-i} = \varnothing$ when $N=1$).

\begin{definition}\label{def:lsbre}
Given a Markov game with horizon $T$, the logistic stochastic best response equilibrium (LSBRE) is a sequence of $T$ stochastic policies $\{\vpi^t\}_{t=1}^T$ constructed by the following process. Consider T Markov chains over $(\mathcal{A}_1 \times \cdots \mathcal{A}_N)^{|\mathcal{S}|}$, where the state of the  t-th Markov chain at step $k$ is $\{z^{t, (k)}_i: \mathcal{S} \rightarrow \mathcal{A}_i\}_{i=1}^N$, with each random variable $z_i^{t, (k)}(s)$ taking values in $\mathcal{A}_i$. For $t \in [T, \ldots, 1]$, we recursively define the the stationary joint distribution $\pi^t(\va|s)$ of the $t$-th Markov chain in terms of $\{\vpi^\ell\}_{\ell=t+1}^T$ as:

For $s^t \in \mathcal{S}, i \sim [1, \cdots, N]$, we update the state of the Markov chain as:
\begin{align}
  z^{t, (k+1)}_i(s^t) \sim  P^t_i(a^t_i|\va^t_{-i}=\vz^{t, (k)}_{-i}(s^t), s^t) =\nonumber\\
  \frac{ \exp\left( \lambda Q_i^{\vpi^{t+1:T}} (s^t, a^t_i, \vz^{t, (k)}_{-i}(s^t)) \right)}{\sum_{a'_{i}} \exp\left( \lambda Q_i^{\vpi^{t+1:T}} (s^t,a_i',\vz^{t, (k)}_{-i}(s^t)\right)}\label{eq:lsbre-conditional}
\end{align}
where parameter $\lambda \in \R^{+}$ controls the level of rationality of the agents, and $\{P^t_i\}_{i=1}^N$ specifies a set of 
conditional distributions.
LSBRE for Markov game is  the sequence of $T$ joint stochastic policies $\{\vpi^t\}_{t=1}^T$. Each joint policy $\vpi^t:\mathcal{S} \rightarrow \mathcal{P}(\mathcal{A}_1 \times \cdots  \times \mathcal{A}_N)$ is given by:
\begin{align}
    \vpi^t(a_1, \cdots, a_N|s^t) = P\left(\bigcap_i \{z^{t, (\infty)}_i(s^t) = a_i \} \right)\label{eq:lsbre-joint}
\end{align}

where the probability is taken with respect to the unique stationary distribution of the t-th Markov chain.
\end{definition}

When the set of conditionals
in Equation (\ref{eq:conditionals}) 
are compatible (in the sense that each conditional can be inferred from the \textit{same} joint distribution \cite{arnold1989compatible}), the above process corresponds to a \textit{Gibbs sampler}, which will converge to a stationary joint distribution $\vpi(\va)$, whose conditional distributions are consistent with the ones used during sampling, namely Equation (\ref{eq:conditionals}). This is the case, for example, if the agents are cooperative, \textit{i.e.}, they share the same reward function $r_i$. In general, $\vpi(\va)$ is the distribution specified by the dependency network~\cite{heckerman2000dependency} defined via conditionals in Equation (\ref{eq:conditionals}). The same argument can be made for the Markov Chains in Definition \ref{def:lsbre} with respect to the conditionals in Equation  (\ref{eq:lsbre-conditional}).

When the set of conditionals in Equation~(\ref{eq:conditionals}) and (\ref{eq:lsbre-conditional}) are incompatible, the procedure is called a \textit{pseudo Gibbs sampler}. As discussed in literatures on dependency networks~\citep{heckerman2000dependency,chen2011gibbs,chen2015behaviour}, when the conditionals are \emph{learned} from a sufficiently large dataset, the pseudo Gibbs sampler asymptotically works well in the sense that the conditionals of the stationary joint distribution are nearly consistent with the conditionals used during sampling. Under some conditions, theoretical bounds on the approximation can be obtained~\citep{heckerman2000dependency}.

\subsection{Trajectory Distributions Induced by LSBRE}
Following \cite{fu2017learning,levine2018reinforcement}, without loss of generality, in the remainder of this paper we consider the case where $\lambda=1$.
First, we note that there is a connection between the notion of LSBRE and maximum causal entropy reinforcement learning \cite{ziebart2010modeling}. Specifically, we can characterize the trajectory distribution induced by LSBRE policies with an energy-based formulation, where the probability of a trajectory increases exponentially as the sum of rewards increases.
Formally, with LSBRE policies, the probability of generating a certain trajectory can be characterized with the following theorem:
\begin{theorem}\label{the:traj}
 Given a joint policy $\{\vpi^t(\va^t|s^t)\}_{t=1}^T$ specified by LSBRE, for each agent $i$, let $\{\vpi^t_{-i}(\va^t_{-i}|s^t)\}_{t=1}^T $ denote other agents' marginal distribution and $\{\pi_i^t(a^t_{i}|\va^t_{-i},s^t)\}_{t=1}^T$ denote agent $i$'s conditional distribution, both obtained from the LSBRE joint policies. 
Then the LSBRE conditional distributions are the optimal solution to the following optimization problem:
\begin{align}
&~~~~~~~~~\min_{\hat{\pi}^{1:T}} \KL (\hat{p}(\tau)||\tilde{p}(\tau))\label{eq:mini-kl}\\
\hat{p}(\tau) = & \left[\eta(s^1) \cdot  \prod_{t=1}^T P(s^{t+1}|s^t,\va^t)  \cdot \vpi^t_{-i}(\va^t_{-i}|s^t)) \right] \cdot \nonumber\\
& \prod_{t=1}^T \hat{\pi}_i^t(a^t_{i}|\va^t_{-i},s^t) \nonumber\\
\tilde{p}(\tau) \propto & \left[\eta(s^1) \cdot  \prod_{t=1}^T P(s^{t+1}|s^t,\va^t)  \cdot \vpi^t_{-i}(\va^t_{-i}|s^t)) \right] \cdot \nonumber\\
  & \exp\left(\sum_{t=1}^T r_{i}(s^t, a^t_i, \va_{-i}^t)\right) \label{eq:lsbre-traj}
\end{align}

\end{theorem}
\begin{proof}
See Appendix~\ref{app:lsbre-traj}.
\end{proof}

Intuitively, for single-shot normal form games, the above statement holds obviously from the definition in Equation (\ref{eq:conditionals}). For Markov games, similar to the process introduced in Definition \ref{def:lsbre}, we can employ a dynamic programming algorithm to find the conditional policies which minimizes Equation (\ref{eq:mini-kl}). Specifically, we first construct the base case of $t=T$ as a normal form game, then recursively construct the conditional policy for each time step $t$, based on the policies from $t+1$ to $T$ that have already been constructed. It can be shown that the constructed optimal policy which minimizes the KL divergence between its trajectory distribution and the trajectory distribution defined in Equation (\ref{eq:lsbre-traj}) corresponds to the set of conditional policies in LSBRE.

\subsection{Multi-Agent Adversarial IRL}\label{sec:ma-airl}
In the remainder of this paper, we assume that the expert policies form a unique LSBRE under some unknown (parameterized) reward functions, according to Definition~\ref{def:lsbre}. 
By adopting LSBRE as the optimality notion, we are able to rationalize the demonstrations by maximizing the likelihood of the expert trajectories with respect to the LSBRE stationary distribution, which is in turn induced by the $\vomega$-parameterized reward functions $\{r_i(s, \va; \omega_i)\}_{i=1}^N$.

The probability of a trajectory $\tau = \{s_t, \va_t\}_{t=1}^T$ generated by LSBRE policies in a Markov game is defined by the following generative process:
\begin{equation}\label{eq:generative}
p(\tau) = \eta(s^1) \cdot \prod_{t=1}^T \vpi^t(\va^t|s^t; \vomega) \cdot \prod_{t=1}^T P(s^{t+1}|s^t,\va^t)
\end{equation}
where $\vpi^t(\va^t|s^t; \vomega)$ are the unique stationary joint distributions for the LSBRE induced by $\{r_i(s, \va; \omega_i)\}_{i=1}^N$.
 The initial state distribution $\eta(s^1)$ and transition dynamics $P(s^{t+1}|s^t,\va^t)$ are specified by the Markov game.

As mentioned in Section~\ref{sec:single-airl}, the MaxEnt IRL framework interprets finding suitable reward functions as maximum likelihood over the expert trajectories in the distribution defined in Equation (\ref{eq:generative}), which can be reduced to:
\begin{equation}\label{eq:mle}
    \max_\vomega \mathbb{E}_{\tau \sim \vpi_E}\left[\sum_{t=1}^T \log \vpi^t(\va^t|s^t; \vomega)\right]
\end{equation}
since the initial state distribution and transition dynamics do not depend on the parameterized rewards.

Note that $\vpi^t(\va^t|s^t)$ in Equation (\ref{eq:mle}) is the joint policy defined in Equation (\ref{eq:lsbre-joint}), whose conditional distributions are given by Equation (\ref{eq:lsbre-conditional}). From Section \ref{sec:lsbre}, we know that given a set of $\vomega$-parameterized reward functions, we are able to characterize the conditional policies $\{\pi^t_i(a_i^t| \va_{-i}^t, s^t)\}_{t=1}^T$ for each agent $i$.
However direct optimization over the joint MLE objective in Equation (\ref{eq:mle}) is intractable, as we cannot obtain a closed form for the stationary joint policy. Fortunately, we are able to construct an asymptotically consistent estimator by approximating the joint likelihood $\vpi^t(\va^t|s^t)$ with a product of the conditionals $\prod_{i=1}\pi_i^t(a^t_i|\va_{-i}^t, s^t)$, which is termed a \textit{pseudolikelihood} \cite{besag1975statistical}.

With the asymptotic consistency property of the \textit{maximum pseudolikelihood estimation} \cite{besag1975statistical, lehmann2006theory}, we have the following theorem:
\begin{theorem}\label{the:pseudo}
Let demonstrations $\tau_1, \ldots, \tau_M$ be independent and identically distributed (sampled from LSBRE induced by some unknown reward functions), and suppose that for all $t \in [1, \ldots, T], a^t_i \in \mathcal{A}_i$, $\pi^t_i(a^t_i|\va^t_{-i}, s^t; \omega_i)$ 
is differentiable with respect to $\omega_i$.
Then, with probability tending to $1$ as $M \to \infty$, the equation
\begin{equation}
    \frac{\partial}{\partial \vomega} \sum_{m=1}^M \sum_{t=1}^T \sum_{i=1}^N \log \pi^t_i(a^{m,t}_i|\va^{m,t}_{-i}, s^{m,t}; \omega_i) = 0
    \label{eq:pseudolikelihood-derivative}
\end{equation}
has a root $\Hat{\vomega}_M$ such that $\Hat{\vomega}_M$ tends to the maximizer of the joint likelihood in Equation~(\ref{eq:mle}).
\end{theorem}
\begin{proof}
See Appendix~\ref{app:pseudolikelihood}.
\end{proof}

Theorem~\ref{the:pseudo} bridges the gap between optimizing the joint likelihood and each conditional likelihood. 
Now we are able to maximize the objective in Equation (\ref{eq:mle}) as:
\begin{equation}\label{eq:policy-mle}
    \mathbb{E}_{\vpi_E} \left[\sum_{i=1}^N \sum_{t=1}^T \frac{\partial}{\partial \vomega} \log \pi^t_i(a^t_i|\va^t_{-i}, s^t; \omega_i) \right]
\end{equation}

To optimize the maximum pseudolikelihood objective in Equation~(\ref{eq:policy-mle}), we can instead optimize the following surrogate loss which is a variational approximation to the psuedolikelihood objective (from Theorem~\ref{the:traj}):
\begin{align*}
    \mathbb{E}_{\vpi_E} \left[\sum_{i=1}^N \sum_{t=1}^T \frac{\partial}{\partial \vomega} r_{i}(s^t, \va^t; \omega_i) \right] - \sum_{i=1}^N \frac{\partial}{\partial \vomega}  \log Z_{\omega_i}
\end{align*}
where $Z_{\omega_i}$ is the partition function of the distribution in Equation~(\ref{eq:lsbre-traj}).
It is generally intractable to exactly compute and optimize the partition function $Z_\omega$, which involves an integral over all trajectories. Similar to GCL \cite{finn2016guided} and single-agent AIRL \cite{fu2017learning}, we employ \textit{importance sampling} to estimate the partition function with adaptive samplers $\vq_\vtheta$. Now we are ready to introduce our practical Multi-Agent Adversarial IRL (MA-AIRL) framework, where we train the $\vomega$-parameterized discriminators as:
\begin{align}\label{eq:gan-discriminator}
\max_{\vomega}~ 
& \mathbb{E}_{\vpi_E} \left[\sum_{i=1}^N \log \frac{\exp(f_{\omega_i}(s, \va))}{\exp(f_{\omega_i}(s, \va)) + q_{\theta_i}(a_i|s)}\right] + \nonumber\\
& \mathbb{E}_{\vq_\vtheta} \left[\sum_{i=1}^N \log \frac{q_{\theta_i}(a_i|s)}{\exp(f_{\omega_i}(s, \va)) + q_{\theta_i}(a_i|s)}\right]
\end{align}
and we train the $\vtheta$-parameterized generators as:
\begin{align}\label{eq:generator-obj}
    \max_\theta~& \mathbb{E}_{\vq_\vtheta} \left[\sum_{i=1}^N \log(D_{\omega_i}(s, \va)) - \log(1-D_{\omega_i}(s,\va))\right] \nonumber\\
    = & \ \mathbb{E}_{\vq_\vtheta} \left[\sum_{i=1}^N f_{\omega_i}(s, \va) - \log(q_{\theta_i}(a_i | s))\right]
\end{align}
Specifically, for each agent $i$, we have a discriminator with a particular structure $\frac{\exp(f_{\omega})}{\exp(f_{\omega}) + q_\theta}$ for a binary classification,
and a generator as an adaptive importance sampler for estimating the partition function. Intuitively, $q_\theta$ is trained to minimize the KL divergence between its trajectory distribution and that induced by the reward functions, for reducing the variance of importance sampling, while $f_\omega$ in the discriminator is trained to estimate the reward function.
At optimality, $f_\omega$ will approximate the advantage function for the expert policy and $q_\theta$ will approximate the expert policy.

\subsection{Solving Reward Ambiguity in Multi-Agent IRL}
For single-agent reinforcement learning, \citeauthor{ng1999policy} shows that 
for any state-only potential function $\phi: \mathcal{S} \to \mathbb{R}$, potential-based reward shaping defined as:
\begin{equation}
    r'(s^t,a^t,s^{t+1}) = r(s^t, a^t, s^{t+1}) + \gamma \Phi(s^{t+1}) - \Phi(s^t) \nonumber
\end{equation}
is a necessary and sufficient condition to guarantee invariance of the optimal policy in both finite and infinite horizon MDPs. In other words, given a set of expert demonstrations, there is a class of reward functions, all of which can explain the demonstrated expert behaviors. Thus without further assumptions, it would be impossible to identify the ground-truth reward that induces the expert policy within this class. Similar issues also exist when we consider multi-agent scenarios. \citeauthor{devlin2011theoretical} show that in multi-agent systems, using the same reward shaping for one or more agents will not alter the set of Nash equilibria. It is possible to extend this result to other solution concepts such as CE and LSBRE. For example, in the case of LSBRE, after specifying the level of rationality $\lambda$, for any $\pi_i \neq \pi^\LSBRE_i$, we have:
\begin{equation}\label{eq:inequal}
    \E_{\pi^{\LSBRE}_i, \vpi^{\LSBRE}_{-i}}[r_i(s, \va)] \geq \E_{\pi_i, \vpi^{\LSBRE}_{-i}}[r_i(s, \va)]
\end{equation}
since each individual LSBRE conditional policy is the optimal solution to the corresponding entropy regularized RL problem (See Appendix~(\ref{app:lsbre-traj})). It can be also shown that any policy that satisfies the inequality in Equation~(\ref{eq:inequal}) will still satisfy the inequality after reward shaping \cite{devlin2011theoretical}.

To mitigate the reward shaping effect and recover reward functions with higher linear correlation to the ground truth reward, as in \cite{fu2017learning}, we further assume the functions $f_{\omega_i}$ in Equation~(\ref{eq:gan-discriminator}) have a specific structure:
\begin{equation}
    f_{\omega_i, \phi_i} (s^t,a^t,s^{t+1}) = g_{\omega_i}(s^t,a^t) + \gamma h_{\phi_i}(s^{t+1}) - h_{\phi_i}(s^t) \nonumber
\end{equation}
where $g_\omega$ is a reward estimator and  $h_\phi$ is a potential function. We summarize the MA-AIRL training procedure in Algorithm~\ref{alg:mairl}. 

\begin{algorithm}[t]
   \caption{Multi-Agent Adversarial IRL}
   \label{alg:mairl}
\begin{algorithmic}
   \STATE {\bfseries Input:} Expert trajectories $\mathcal{D}_E = \{\tau^E_j\}$; Markov game as a black box with parameters $(N, \bm{\mathcal{S}}, \bm{\mathcal{A}}, \eta, \bm{P}, \gamma)$
   \STATE Initialize the parameters of policies $\vq$, reward estimators $\vg$ and potential functions $\vh$ with $\vtheta, \vomega, \vphi$. 
   \REPEAT
   \STATE Sample trajectories $\mathcal{D}_\pi = \{\tau_j\}$ from $\bm{\pi}$:
   \STATE ~~~$s^1 \sim \eta(s),  \va^t \sim \bm{\pi}(\va^t|s^t), s^{t+1} \sim P(s^{t+1}|s^t, \va^t)$
   \STATE Sample state-action pairs $\mathcal{X}_\pi$, $\mathcal{X}_E$ from $\mathcal{D}_\pi$, $\mathcal{D}_E$.
   \FOR{$i= 1, \dots$, N}
   \STATE Update
   $\omega_i$, $\phi_i$ to increase the objective in Eq.~\ref{eq:gan-discriminator}:
   \STATE {\small $\mathbb{E}_{\mathcal{X}_E} [\log D(s, a_i, s')] + \mathbb{E}_{\mathcal{X}_\pi} [\log(1-D(s, a_i, s'))]$}
   \ENDFOR
   \FOR{$i= 1, \dots$, N}
   \STATE Update reward estimates $\hat{r}_i(s, a_i, s')$ with $g_{\omega_i}(s, a_i)$.
   or ($\log D(s, a_i, s') - \log(1-D(s, a_i, s'))$)
   \STATE Update $\theta_i$ with respect to $\hat{r}_i(s, a_i, s')$.
   \ENDFOR
   \UNTIL{Convergence}
   \STATE {\bfseries Output:} Learned policies $\vpi_\vtheta$ and reward functions $\vg_\vomega$.
\end{algorithmic}
\end{algorithm}

\section{Related Work}
A vast number of methods and paradigms have been proposed for single-agent imitation learning and inverse reinforcement learning.
However, multi-agent scenarios are less commonly investigated, and most existing works assume specific reward structures. These include fully cooperative games~\cite{barrett2017making,le2017coordinated,vsovsic2017inverse,bogert2014multi}, two player zero-sum games~\citep{lin2014multi}, and rewards as linear combinations of pre-specified features~\citep{reddy2012inverse,waugh2013computational}. Recently, \citeauthor{song2018multi} proposed MA-GAIL, a multi-agent extension of GAIL which works on general Markov games.

While both MA-AIRL and MA-GAIL are based on adversarial training, the methods are inherently different. MA-GAIL is based on the notion of Nash equilibrium, and is motivated via a specific choice of Lagrange multipliers for a constraint optimization problem.
MA-AIRL, on the other hand, is derived from MaxEnt RL and LSBRE, and aims to obtain an MLE solution for the joint trajectories; we connect this with a set of conditionals via pseudolikelihood, which are then solved with the adversarial reward learning framework. From a reward learning perspective, the discriminators' outputs in MA-GAIL will converge to uninformative uniform distribution, while MA-AIRL allows us to recover reward functions from the optimal discriminators.

\section{Experiments}
We seek to answer the following questions via empirical evaluation:
(1) Can MA-AIRL efficiently recover the expert policies for each individual agent from the expert demonstrations (policy imitation)?
(2) Can MA-AIRL effectively recover the underlying reward functions, for which the expert policies form a LSBRE (reward recovery)? 

\textbf{Task Description}~~~To answer these questions, we evaluate our MA-AIRL algorithm on  a series of simulated particle environments~\citep{lowe2017multi}. Specifically, we consider the following scenarios: \textit{cooperative navigation}, where three agents cooperate through physical actions to reach three landmarks; \textit{cooperative communication}, where two agents, a speaker and a listener, cooperate to navigate to a particular landmark; and \textit{competitive keep-away}, where one agent tries to reach a target landmark, while an adversary, without knowing the target a priori, tries to infer the target from the agent's behaviors and prevent it from reaching the goal through physical interactions. 

In our experiments, for generality, the learning algorithms will not leverage any prior knowledge on the types of interactions (cooperative or competitive). Thus for all the tasks described above, the learning algorithms will take a decentralized form and we will not utilize additional reward regularization, besides penalizing the $\ell$2 norm of the reward parameters to mitigate overfitting \cite{ziebart2010modeling, kalakrishnan2013learning}.

\textbf{Training Procedure}~~~In the simulated environments, we have access to the ground-truth reward functions, which enables us to accurately evaluate the quality of both recovered policies and reward functions. We use a multi-agent version of ACKTR~\citep{wu2017scalable,song2018multi}, an efficient model-free policy gradient algorithm for training the experts as well as the adaptive samplers in MA-AIRL. The supervision signals for the experts come from the ground-truth rewards, while the reward signals for the adaptive samplers come from the discriminators. Specifically, we first obtain expert policies induced by the ground-truth rewards, then we use them to generate demonstrations, from which the learning algorithms will try to recover the policies as well as the underlying reward functions. We compare MA-AIRL against the state-of-the-art multi-agent imitation learning algorithm, MA-GAIL \cite{song2018multi}, which is a generalization of GAIL to Markov games. Following \cite{li2017infogail, song2018multi}, we use behavior cloning to pretrain MA-AIRL and MA-GAIL to reduce sample complexity for exploration, and we use 200 episodes of expert demonstrations, each with 50 time steps, which is close to the amount of time steps used in \cite{ho2016generative}\footnote{The codebase for this work can be found at \url{https://github.com/ermongroup/MA-AIRL}.}.

\subsection{Policy Imitation}

Although MA-GAIL achieved superior performance compared with behavior cloning \cite{song2018multi}, it only aims to recover policies via distribution matching.
Moreover, the training signal for the policy will become less informative as training progresses; according to~\citep{goodfellow2014generative} with infinite data and computational resources the discriminator outputs will converge to 0.5 for all state-action pairs, which could potentially hinder the robustness of the policy towards the end of training. 
To empirically verify our claims, we compare the quality of the learned policies in terms of the expected return received by each agent. 

In the cooperative environment, we directly use the ground-truth rewards from the environment as the oracle metric, since all agents share the same reward. In the competitive environment, we follow the evaluation procedure in~\citep{song2018multi}, where we place the experts and learned policies in the same environment. A learned policy is considered ``better'' if it receives a higher expected return while its opponent receives a lower expected return. The results for cooperative and competitive environments are shown in Tables~\ref{tab:cooperative-rews} and~\ref{tab:competitive-rews} respectively. MA-AIRL consistently performs better than MA-GAIL in terms of the received reward in all the considered environments, suggesting superior imitation learning capabilities to the experts. 

\begin{table}[t]
\centering
\caption{Expected returns in cooperative tasks. Mean and variance are taken across different random seeds used to train the policies.}
\vspace{5pt}
\label{tab:cooperative-rews}
\begin{tabular}{c|c|c}
\toprule
Algorithm & Nav. ExpRet & Comm. ExpRet\\
\midrule
Expert & -43.195 $\pm$ 2.659 & -12.712 $\pm$ 1.613\\ 

Random & -391.314 $\pm$ 10.092 & -125.825 $\pm$ 3.4906\\
\midrule
MA-GAIL & -52.810 $\pm$ 2.981 & -12.811 $\pm$ 1.604\\

MA-AIRL & \textbf{-47.515} $\pm$ 2.549 & \textbf{-12.727} $\pm$ 1.557\\
\bottomrule
\end{tabular}
\vspace{-12pt}
\end{table}

\begin{table}[t]
\centering
\caption{Expected returns of the agents in competitive task. 
Agent~\#1 represents the agent trying to reach the target and Agent~\#2 represents the adversary. Mean and variance are taken across different random seeds.
}
\vspace{5pt}
\label{tab:competitive-rews}
\begin{tabular}{c|c|c}
\toprule
Agent~\#1 & Agent~\#2 & Agent~\#1 ExpRet\\
\midrule
Expert & Expert & -6.804 $\pm$ 0.316\\
\midrule
MA-GAIL & Expert & -6.978 $\pm$ 0.305\\

MA-AIRL & Expert & \textbf{-6.785} $\pm$ 0.312\\
\midrule
Expert & MA-GAIL & -6.919 $\pm$ 0.298\\

Expert & MA-AIRL & \textbf{-7.367}$\pm$ 0.311\\
\bottomrule
\end{tabular}
\vspace{-12pt}
\end{table}

\subsection{Reward Recovery}

The second question we seek to answer is concerned with the reward recovering problem as in inverse reinforcement learning: is the algorithm able to recover the ground truth reward functions with expert demonstrations being the only source of supervision? To answer this question, we evaluate the statistical correlations between the ground truth rewards (which the learning algorithms have no access to) and the inferred rewards for the same state-action pairs. 

Specifically, we consider two types of statistical correlations:
\textit{Pearson's correlation coefficient} (PCC), which measures the linear correlation between two random variables; and \textit{Spearman's rank correlation coefficient} (SCC), which measures the statistical dependence between the rankings of two random variables. 
Higher SCC suggests that two reward functions have higher monotonic relationships 
and higher PCC suggests higher linear correlations. For each trajectory, we compare the ground-truth 
return from the environment with the supervision signals from the discriminators, which correspond to $g_\omega$ in MA-AIRL and $\log(D_\omega)$ in MA-GAIL.

Tables~\ref{tab:cooperative-correlation} and~\ref{tab:competitive-correlation} provide the SCC and PCC statistics for cooperative and competitive environments respectively. In the cooperative case, compared to MA-GAIL, MA-AIRL achieves a much higher PCC and SCC, which could facilitate policy learning. The statistical correlations between reward signals gathered from discriminators for each agent are also quite high, suggesting that while we do not reveal the agents are cooperative, MA-AIRL is able to discover high correlations between the agents' reward functions. In the competitive case, the reward functions learned by MA-AIRL also significantly outperform MA-GAIL in terms of SCC and PCC statistics. In Figure~\ref{fig:pcc}, we further show the changes of PCC statistics with respect to training time steps for MA-GAIL and MA-AIRL. The reward functions recovered by MA-GAIL initially have a high correlation with the ground truth, yet that dramatically decreases as training continues, whereas the functions learned by MA-AIRL maintains a high correlation throughout the course of training, which is in line with the theoretical analysis that in MA-GAIL, reward signals from the discriminators will become less informative towards convergence.

\begin{table}[t]
\centering
\caption{Statistical correlations between the learned reward functions and the ground-truth rewards in cooperative tasks. Mean and variance are taken across $N$ independently learned reward functions for $N$ agents.}
\vspace{5pt}
\label{tab:cooperative-correlation}
\begin{tabular}{c|c|c|c}
\toprule
  Task & Metric & MA-GAIL & MA-AIRL \\ \midrule

\multirow{2}{*}{\shortstack{Nav. \tiny}} & SCC & 0.792 $\pm$ 0.085 & \textbf{0.934} $\pm$ 0.015\\

& PCC & 0.556 $\pm$ 0.081 & \textbf{0.882} $\pm$ 0.028\\

\midrule
\multirow{2}{*}{\shortstack{Comm. \tiny}} & SCC & 0.879 $\pm$ 0.059 & \textbf{0.936} $\pm$ 0.080\\

& PCC & 0.612 $\pm$ 0.093 & \textbf{0.848} $\pm$ 0.099\\
\bottomrule
\end{tabular}
\vspace{-12pt}
\end{table}

\begin{table}[t]
\centering
\caption{Statistical correlations between the learned reward functions and the ground-truth rewards in competitive task.} 
\vspace{5pt}
\label{tab:competitive-correlation}
\begin{tabular}{c|cc}
\toprule
Algorithm & MA-GAIL & MA-AIRL\\
\midrule
SCC \#1 & 0.424 & \textbf{0.534}\\

SCC \#2 & 0.653 & \textbf{0.907}\\

\midrule
Average SCC & 0.538 & \textbf{0.721}\\

\midrule
PCC \#1 & 0.497 & \textbf{0.720}\\

PCC \#2 & 0.392 & \textbf{0.667}\\
\midrule
Average PCC & 0.445 & \textbf{0.694}\\
\bottomrule
\end{tabular}
\vspace{-12pt}
\end{table}

\begin{figure}[th]
\centering
\includegraphics[width=0.4\textwidth]{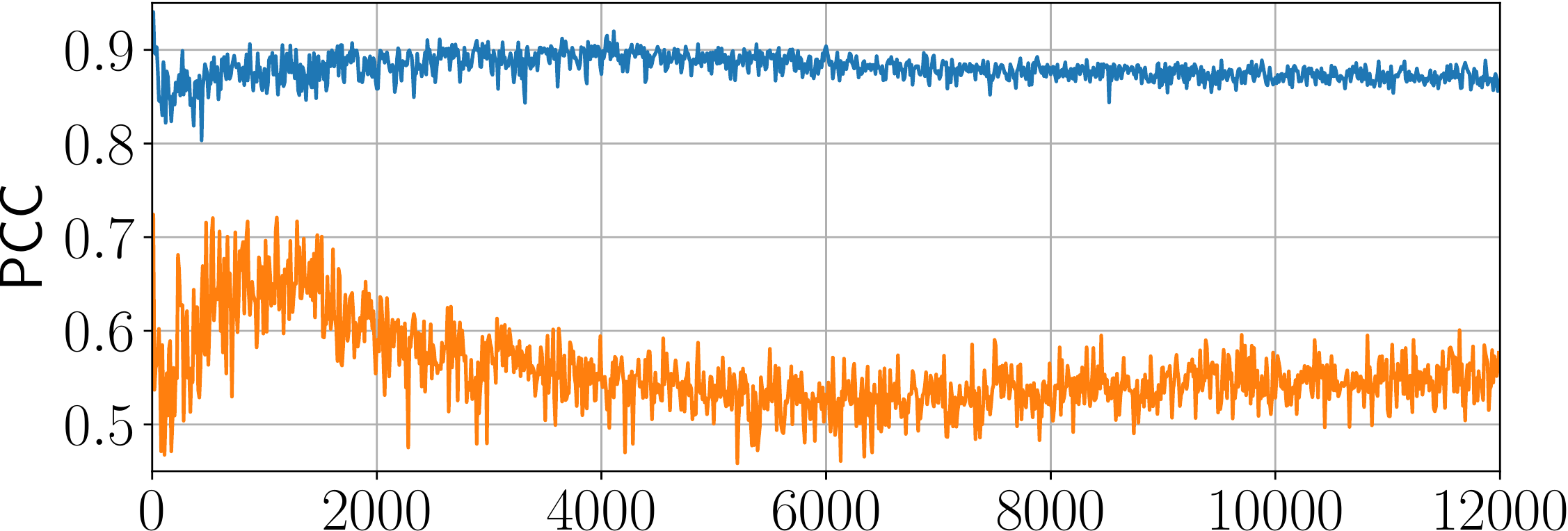}
\vspace{-5pt}
\caption{PCC w.r.t. the training epochs in cooperative navigation, with MA-AIRL (blue) and MA-GAIL (orange).}\label{fig:pcc}
\vspace{-8pt}
\end{figure}

\section{Discussion and Future Work}
We propose MA-AIRL, the first multi-agent MaxEnt IRL framework that is effective and scalable to Markov games with high-dimensional state-action space and unknown dynamics. We derive our algorithm based on a solution concept termed LSBRE and we employ maximum pseudolikelihood estimation to achieve tractability.
Experimental
results demonstrate that MA-AIRL is able to imitate expert behaviors in high-dimensional complex environments, as well as learn reward functions that are highly correlated with the ground truth rewards. An exciting avenue for future work is to include reward regularization to mitigate overfitting and leverage prior knowledge of the task structure.

\section*{Acknowledgments}
This research was supported by Toyota Research Institute, NSF (\#1651565, \#1522054, \#1733686), ONR (N00014-19-1-2145), AFOSR (FA9550-19-1-0024), Amazon AWS, and Qualcomm.

\bibliography{ref}
\bibliographystyle{icml2019}

\appendix
\onecolumn
\section{Appendix}
\subsection{Trajectory Distribution Induced by Logistic Stochastic Best Response Equilibrium}\label{app:lsbre-traj}

Let $\{\vpi^t_{-i}(\va^t_{-i}|s^t)\}_{t=1}^T$ denote other agents' marginal LSBRE policies, and $\{\hat{\pi}^t_i (a_i^t|\va_{-i}^t, s^t)\}_{t=1}^T$ denote agent $i$'s conditional policy. With chain rule, the induced trajectory distribution is given by:
\begin{align}
\hat{p}(\tau) = \left[\eta(s^1) \cdot  \prod_{t=1}^T P(s^{t+1}|s^t,\va^t)  \cdot \vpi^t_{-i}(\va^t_{-i}|s^t)) \right] \cdot \prod_{t=1}^T \hat{\pi}^t_i (a^t_i|\va^t_{-i}, s^t)
\end{align}
Suppose the desired distribution is given by:
\begin{align}
p(\tau) \propto \left[\eta(s^1) \cdot  \prod_{t=1}^T P(s^{t+1}|s^t,\va^t)  \cdot \vpi^t_{-i}(\va^t_{-i}|s^t)) \right] \cdot \exp\left(\sum_{t=1}^T r_{i}(s^t, a^t_i, \va_{-i}^t)\right)\label{eq:desired}
\end{align}

Now we will shown that the optimal solution to the following optimization problem correspond to the LSBRE conditional policies:
\begin{align}
    &\min_{\hat{\pi}_i^{1:T}} \KL(\hat{p}(\tau)|| p(\tau)) \label{eq:proof-optimization}
\end{align}

The optimization problem in Equation~(\ref{eq:proof-optimization}) is equivalent to (the partition function of the desired distribution is a constant with respect to optimized policies):
\begin{align}
    \max_{\hat{\pi}_i^{1:T}} ~& \mathbb{E}_{\tau \sim \hat{p}(\tau)} \left[ \log \eta(s^1) + \sum_{t=1}^T (\log P(s^{t+1}|s^t,\va^t) + \log \vpi^t_{-i}(\va^t_{-i}|s^t) + r_i(s^t, \va^t)) - \right. \nonumber\\
    &~~~~~~~~~~~~~~~~~\left. \log \eta(s^1) - \sum_{t=1}^T (\log P(s^{t+1}|s^t,\va^t) + \log \vpi^t_{-i}(\va^t_{-i}|s^t) + \log \hat{\pi}^t_i (a^t_i|\va^t_{-i}, s^t)) \right] \nonumber\\
    =~& \mathbb{E}_{\tau \sim \hat{p}(\tau)} \left[ \sum_{t=1}^T r_i(s^t, \va^t) - \log \hat{\pi}^t_i (a^t_i|\va^t_{-i}, s^t) \right] 
    = \sum_{t=1}^T \mathbb{E}_{(s^t, \va^t) \sim \hat{p}(s^t, \va^t)} [r_i(s^t, \va^t) - \log \hat{\pi}^t_i (a^t_i|\va^t_{-i}, s^t)]
\end{align}
To maximize this objective, we can use a dynamic programming procedure. Let us first consider the base case of optimizing $\hat{\pi}^T_i (a_i^T|\va_{-i}^T, s^T)$:
\begin{align}\label{eq:base-case-kl}
    & \mathbb{E}_{(s^T, \va^T) \sim \hat{p}(s^T, \va^T)} [r_i(s^T, \va^T) - \log \hat{\pi}^T_i (a_i^T|\va_{-i}^T)] = \nonumber\\
    & \mathbb{E}_{s^T \sim \hat{p}(s^T), \va^T_{-i} \sim \vpi^T_{-i}(\cdot|s^T)} \left[- \KL \left(\hat{\pi}^T_i(a_i^T|\va_{-i}^T, s^T) || \frac{\exp(r_i(s^T, a_i^T, \va_{-i}^T))}{\exp(V_i(s^T, \va_{-i}^T))}\right) + V_i(s^T, \va_{-i}^T) \right]
\end{align}
where $\exp(V_i(s^T,\va_{-i}^T))$ is the partition function and $V_i(s^T,\va_{-i}^T) = \log \sum_{a'_i} \exp(r_i(s^T, a'_i, \va_{-i}^T))$. The optimal policy is given by:
\begin{align}\label{eq:base-case}
    \pi^T_i(a^T_i|\va^T_{-i}, s^T) = \exp(r_i(s^T, a^T_i, \va^T_{-i}) - V_i(s^T,\va_{-i}^T))
\end{align}
With the optimal policy in Equation~(\ref{eq:base-case}), Equation~(\ref{eq:base-case-kl}) is equivalent to (with the KL divergence being zero):
\begin{align}
    \mathbb{E}_{(s^T, \va^T) \sim \hat{p}(s^T, \va^T)} [r_i(s^T, \va^T) - \log \hat{\pi}^T_i (a_i^T|\va_{-i}^T)] = \mathbb{E}_{s^T \sim \hat{p}(s^T), \va^T_{-i} \sim \vpi^T_{-i}(\cdot|s^T)} [V_i(s^T, \va_{-i}^T)]
\end{align}

Then recursively, for a given time step $t$, $\hat{\pi}^t_i (a_i^t|\va_{-i}^t, s^t)$ must maximize:
\begin{align}
    &\mathbb{E}_{(s^t, \va^t) \sim \hat{p}(s^t, \va^t)}\left[r_i(s^t, \va^t) - \log \hat{\pi}^t_i (a_i^t|\va_{-i}^t) + \mathbb{E}_{s^{t+1} \sim P(\cdot|s^t, \va^t), \va_{-i}^{t+1} \sim \vpi_{-i}^{t+1}(\cdot|s^{t+1})}[V_i^{\vpi^{t+2:T}}(s^{t+1}, \va_{-i}^{t+1})]\right]= \\
    &\mathbb{E}_{s^t \sim \hat{p}(s^t), \va^t_{-i} \sim \vpi^t_{-i}(\cdot|s^t)} \left[ -\KL \left(\hat{\pi}^t_i(a_i^t|\va_{-i}^t, s^t) || \frac{\exp(Q_i^{\vpi^{t+1:T}}(s^t, a_i^t, \va_{-i}^t))}{\exp(V_i^{\vpi^{t+1:T}}(s^t, \va_{-i}^t))}\right) + V_i^{\vpi^{t+1:T}}(s^t, \va^t_{-i}) \right]\label{eq:intermediate-obj}
\end{align}
where we define:
\begin{align}
    Q_i^{\vpi^{t+1:T}}(s^t, \va^t) &= r_i(s^t, \va^t) + \mathbb{E}_{s^{t+1}\sim p(\cdot|s^t, \va^t)} \left[ \mathcal{H}(\pi_i^{t+1}(\cdot|s^{t+1})) + \mathbb{E}_{\va_{-i}^{t+1} \sim \vpi_{-i}^{t+1}(\cdot|s^{t+1})}[V_i(s^{t+1}, \va_{-i}^{t+1})] \right]\\
    V_i^{\vpi^{t+1:T}}(s^t, \va_{-i}^t) &= \log \sum_{a'_i} \exp(Q_i^{\vpi^{t+1:T}}(s^t, a'_i, \va_{-i}^t))
\end{align}
The optimal policy to Equation~(\ref{eq:intermediate-obj}) is given by:
\begin{align}
    \pi^t_i(a_i^t|\va_{-i}^t, s^t) = \exp(Q_i^{\vpi^{t+1:T}}(s^t, \va^t) - V_i^{\vpi^{t+1:T}}(s^t, \va_{-i}^t))
\end{align}
which is exactly the set of conditional distributions used to produce LSBRE (Definition~\ref{def:lsbre}).

\subsection{Maximum Pseudolikelihood Estimation for LSBRE}\label{app:pseudolikelihood}
Theorem~\ref{the:pseudo} strictly follows the asymptotic consistency property of maximum pseudolikelihood estimation \cite{lehmann2006theory,dawid2014theory}. For simplicity, we will show the proof for normal form games and similar to Appendix~\ref{app:lsbre-traj}, the extension to Markov games can be proved by induction.

Consider a normal form game with $N$ players and reward functions $\{r_i(\va; \omega_i)\}_{i=1}^N$. Suppose the expert demonstrations $\mathcal{D} = \{(a_1, \ldots, a_N)^m\}_{m=1}^M$ are generated by $\vpi(\va; \vomega^*)$, where $\vomega^*$ denotes the true value of the parameters. The pseudolikelihood objective we want to maximize is given by:
\begin{align}
\ell_\PL (\vomega) &= \frac{1}{M} \sum_{m=1}^M \sum_{i=1}^N \log \pi_i(a_i^m|\va_{-i}^m; \omega_i)= \frac{1}{M} \sum_{m=1}^M \sum_{i=1}^N \log \frac{\exp(r_i(a_i^m, \va_{-i}^m; \omega_i))}{\sum_{a'_i} \exp(r_i(a'_i, \va_{-i}^m; \omega_i))}\\
&= \frac{1}{M} \sum_{m=1}^M \sum_{i=1}^N r_i(a_i^m, \va^m_{-i}; \omega_i) - \frac{1}{M} \sum_{m=1}^M \sum_{i=1}^N \log Z(\va^m_{-i}; \omega_i)\\
&= \sum_{i=1}^N \sum_{\va} p_\mathcal{D} (\va) r_i(a_i, \va_{-i}; \omega_i) - \sum_{i=1}^N \sum_{\va_{-i}} p_\mathcal{D} (\va_{-i}) \log Z(\va_{-i}; \omega_i)
\end{align}
where $p_\mathcal{D}$ is the empirical data distribution and $Z(\va_{-i}; \omega_i)$ is the partition function.

Take derivatives of $\ell_\PL (\vomega)$:
\begin{align}
    \frac{\partial}{\partial \vomega} \ell_\PL (\vomega) &= \sum_{i=1}^N \sum_{\va} p_\mathcal{D} (\va) \frac{\partial}{\partial \vomega} r_i(a_i, \va_{-i}; \omega_i) - \sum_{i=1}^N \sum_{\va_{-i}} p_\mathcal{D} (\va_{-i}) \frac{1}{Z(\va_{-i}; \omega_i)} \frac{\partial}{\partial \vomega} Z(\va_{-i}; \omega_i) \\
    &= \sum_{i=1}^N \sum_{\va} p_\mathcal{D} (\va) \frac{\partial}{\partial \vomega} r_i(a_i, \va_{-i}; \omega_i) - \sum_{i=1}^N \sum_{\va_{-i}} p_\mathcal{D} (\va_{-i}) \sum_{a_i} \frac{\exp(r_i(a_i, \va_{-i}; \omega_i))}{Z(\va_{-i}; \omega_i)} \frac{\partial}{\partial \vomega} r_i(a_i, \va_{-i}; \omega_i) \\
    &= \sum_{i=1}^N \sum_{\va} p_\mathcal{D} (\va) \frac{\partial}{\partial \vomega} r_i(a_i, \va_{-i}; \omega_i) - \sum_{i=1}^N \sum_{\va_{-i}} p_\mathcal{D} (\va_{-i}) \sum_{a_i} \pi_i(a_i|\va_{-i}; \omega_i) \frac{\partial}{\partial \vomega} r_i(a_i, \va_{-i}; \omega_i) \label{eq:pseudo-gradient}
\end{align}
When the sample size $m \to \infty$, Equation~(\ref{eq:pseudo-gradient}) is equivalent to:
\begin{align}
\frac{\partial}{\partial \vomega} \ell_\PL (\vomega) &= \sum_{i=1}^N \sum_{\va} p (\va; \vomega^*) \frac{\partial}{\partial \vomega} r_i(a_i, \va_{-i}; \omega_i) - \sum_{i=1}^N \sum_{\va_{-i}} p (\va_{-i}; \vomega^*) \sum_{a_i} \pi_i(a_i|\va_{-i}; \omega_i) \frac{\partial}{\partial \vomega} r_i(a_i, \va_{-i}; \omega_i)\\
&= \sum_{i=1}^N \sum_{\va_{-i}} p (\va_{-i}; \vomega^*) \sum_{a_i} (p(a_i|\va_{-i}; \vomega^*) - \pi_i(a_i|\va_{-i}; \omega_i)) \frac{\partial}{\partial \vomega} r_i(a_i, \va_{-i}; \omega_i)\label{eq:pseudo-final}
\end{align}

When $\vomega = \vomega^*$, the gradients in Equation~(\ref{eq:pseudo-final}) will be zero.

\end{document}